\documentclass{article}

\usepackage{microtype}
\usepackage{booktabs} 

\usepackage[colorlinks=true, citecolor=blue]{hyperref}

\usepackage[preprint]{icml2026}

\PassOptionsToPackage{numbers,compress,sort&compress}{natbib}
\usepackage{amsmath}
\usepackage{amssymb}
\usepackage{mathtools}
\usepackage{amsthm}
\usepackage{pifont}
\usepackage{multicol,multirow}
\usepackage{color,xcolor,colortbl}
\usepackage{arydshln}
\usepackage{natbib}
\usepackage{enumitem}
\usepackage{threeparttable}
\usepackage{subfigure}
\usepackage{subcaption}
\usepackage{graphicx}
\usepackage{bbding} 
\usepackage{setspace}
\usepackage{algorithm}
\usepackage{algorithmic}
\usepackage[misc]{ifsym}
\usepackage{xspace}

\usepackage[capitalize,noabbrev]{cleveref}

\usepackage[textsize=tiny]{todonotes}
\usepackage{fontawesome5}

\definecolor{forestgreen}{RGB}{79,173,91}
\definecolor{forestyellow}{RGB}{245,195,66}
\definecolor{myblue}{rgb}{0.82, 0.94, 0.75}
\definecolor{mygold}{rgb}{1, 0.92, 0.56}
\definecolor{mylightblue}{rgb}{0.70, 0.83, 0.96}
\definecolor{mylightyellow}{rgb}{0.96, 0.88, 0.49}
\definecolor{mylightpink}{rgb}{0.93, 0.79, 0.80}

\theoremstyle{plain}
\newtheorem{theorem}{Theorem}[section]

\theoremstyle{definition}

\theoremstyle{remark}

\newcommand{\method}{UARM\xspace}

\usepackage[most]{tcolorbox}

\tcbuselibrary{breakable} 
\tcbuselibrary{skins}
\newtcolorbox[
    use counter=tcboxcounter,number within=section
]{mybox}[3][]{
    left=3pt,
    right=4pt,
    breakable,
    enhanced,
    title=#2 \thetcbcounter: #3,
    #1
}

\makeatletter
\renewcommand{\printAffiliationsAndNotice}[1]{\global\icml@noticeprintedtrue%
  \stepcounter{@affiliationcounter}%
  {\let\thefootnote\relax\footnotetext{\hspace*{-\footnotesep}\ificmlshowauthors #1\fi%
      \forloop{@affilnum}{1}{\value{@affilnum} < \value{@affiliationcounter}}{
        \textsuperscript{\arabic{@affilnum}}\ifcsname @affilname\the@affilnum\endcsname%
          \csname @affilname\the@affilnum\endcsname%
        \else
          {\bf AUTHORERR: Missing \textbackslash{}icmlaffiliation.}
        \fi
      }.%
      \ifdefined\icmlcorrespondingauthor@text
         { }Correspondence to: \icmlcorrespondingauthor@text.
      \fi
      
      \ \\
      \Notice@String
    }
  }
}
\makeatother

\icmltitlerunning{Uncertainty-Aware Reward Modeling for Stable RLHF}

\begin{document}

\twocolumn[
  \icmltitle{Uncertainty-Aware Reward Modeling for Stable RLHF}
  
  \icmlsetsymbol{corr}{\faEnvelope[regular]}
  
  \begin{icmlauthorlist}
    \icmlauthor{Licheng Pan}{zju,red}
    \icmlauthor{Haocheng Yang}{nus}
    \icmlauthor{Haoxuan Li}{pku}
    \icmlauthor{Yichen Sun}{zju}
    \icmlauthor{Yunsheng Lu}{zju}
    \icmlauthor{Shijian Wang}{red} \\
    \icmlauthor{Lei Shen}{red}
    \icmlauthor{Yuan Lu}{red}
    \icmlauthor{Zhixuan Chu}{zju}
    \icmlauthor{Hao Wang}{zju,red}
  \end{icmlauthorlist}

  \icmlaffiliation{zju}{Zhejiang University}
  \icmlaffiliation{red}{Xiaohongshu Inc}
  \icmlaffiliation{pku}{Peking University}
  \icmlaffiliation{nus}{National University of Singapore}


  \icmlkeywords{Large Language Model, Reward Modeling, Optimal Transport, Noisy Preference}

  \vskip 0.3in
]



\printAffiliationsAndNotice{}  

\begin{abstract}
Reinforcement learning from human feedback (RLHF) aligns large language models by training reward models on preference data and optimizing policies to maximize predicted rewards. However, this pipeline faces two fundamental challenges: \ding{182} \textbf{reward models cannot signal when their predictions are unreliable}, since they usually act as deterministic point estimators; and \ding{183} \textbf{modern group-based policy optimization can amplify unreliable reward signals}, as exemplified by GRPO's uniform treatment of rewards during advantage computation. As policies explore increasingly diverse responses, these two limitations create a critical vulnerability: unreliable reward estimates may be granted disproportionate influence, triggering severe reward hacking. We propose \textbf{Uncertainty-Aware Reward Modeling (UARM)}, which equips reward models with calibrated uncertainty via quantile-based conformal prediction and reweights GRPO advantages through heteroscedastic variance decomposition. Experiments across HelpSteer, UltraFeedback, and PKU-SafeRLHF demonstrate that UARM significantly improves reward model calibration, reduces reward hacking, and enhances downstream alignment quality compared to standard GRPO and uncertainty-agnostic baselines.
\end{abstract}

\section{Introduction}

Reinforcement learning from human feedback~\citep{christiano2017deep,rlhf} has emerged as the dominant paradigm for aligning large language models with human values and preferences. In this framework, a reward model is first trained on pairwise preference data—typically modeled via the Bradley-Terry comparison model~\citep{bradleyterry}—to proxy human judgment, and the policy is then optimized to maximize the predicted reward~\citep{dong2024rlhf}. Recent state-of-the-art systems, from GPT-4~\citep{achiam2023gpt} to DeepSeek-R1~\citep{guo2025deepseek} and Gemini~\citep{comanici2025gemini}, rely heavily on this pipeline to produce helpful, harmless, and honest responses. Yet a fundamental question remains unresolved: \textit{when reward models are uncertain about their predictions, should policies trust them unconditionally?}

This question reveals two central challenges that become especially consequential in modern group-based policy optimization methods such as Group Relative Policy Optimization (GRPO)~\citep{guo2025deepseek,gspo}.
\ding{182} \textbf{Reward models cannot signal when their predictions are unreliable.} Current reward models are deterministic point estimators: they output a single scalar score for each prompt-response pair, with no indication of whether that score reflects a confident judgment or an unreliable guess~\citep{lambert2025rewardbench}. As policies evolve during training, they inevitably generate responses that lie outside the reward model's training distribution---responses the model evaluates with high uncertainty. Without any signal of this uncertainty, the policy treats all reward estimates as equally trustworthy, even when some are fundamentally unreliable. This blind trust creates a vulnerability: the policy may aggressively optimize toward responses the reward model scores highly but uncertainly, leading to misalignment and reward hacking~\citep{amodei2016concrete,gao2023scaling,skalse2022defining}.

\ding{183} \textbf{Group-based advantage standardization can amplify the least trustworthy samples.} In GRPO, advantages are standardized uniformly within each rollout group. By treating all reward signals as equally reliable, this standardization procedure can amplify exactly the samples the reward model is least certain about: a confusing response that receives a spuriously extreme reward skews the group mean and variance, and after standardization is assigned a disproportionately large advantage. Meanwhile, genuinely high-quality responses are pushed below the distorted mean and under-rewarded. This uniform treatment thus triggers severe reward hacking, steering policies toward unreliable signals and away from truly aligned behavior~\citep{fu2025reward,InfoRM,liu2024rrm}.

Prior work has explored uncertainty quantification for neural networks, but their application to stable RLHF remains limited. Ensemble-based methods~\citep{deep_ensemble,mcdropout} provide uncertainty estimates by training multiple models or performing stochastic forward passes, but incur prohibitive computational overhead for large language models deployed in online RLHF~\citep{coste2023reward,eisenstein2023helping}. Conformal prediction~\citep{scp,cqr,lei2018distribution,shafer2008tutorial} offers distribution-free coverage guarantees for uncertainty intervals, but classical variants focus on marginal coverage and are not readily adapted to the conditional, sample-specific reliability signals needed for reweighting advantages in reinforcement learning~\citep{wcp,aci}. Heteroscedastic modeling techniques~\citep{mve,der} can capture varying uncertainty across inputs, yet have not been integrated into RLHF's advantage computation. Consequently, no prior work systematically addresses the gap between calibrated uncertainty quantification in reward models and its direct integration into policy optimization to prevent reward hacking.

We propose \textbf{Uncertainty-Aware Reward Modeling (UARM)}, a unified framework that equips reward models with calibrated uncertainty and leverages it to stabilize GRPO. Our approach operates in two phases. In the \textit{offline phase}, we train the reward model as a quantile regression estimator~\citep{pinball} that outputs multiple conditional quantiles of the reward distribution, from which we derive both a point estimate (the median) and a prediction interval whose width captures per-sample uncertainty. We calibrate these intervals on a held-out set via a conformal prediction procedure~\citep{cqr}, achieving conditional coverage guarantees (Theorem~\ref{thm:conditional}) that ensure the interval width faithfully reflects the model's confidence. In the \textit{online phase}, we reinterpret the interval width as observation noise under a heteroscedastic model and decompose the observed reward variance into signal and noise components. This variance decomposition yields a sample-specific reliability weight for each rollout, which we use to construct a heteroscedastic advantage that provably down-weights high-uncertainty samples without requiring costly ensemble evaluations. The entire pipeline integrates seamlessly into GRPO with negligible computational overhead.

Our contributions are summarized as follows:
\begin{itemize}[leftmargin=*]
    \item We develop a quantile-based conformal reward model that provides calibrated per-sample uncertainty estimates with theoretical coverage guarantees.
    \item We introduce a heteroscedastic advantage reweighting scheme that uses uncertainty to suppress unreliable samples in GRPO's standardization.
    \item We conduct comprehensive experiments across three preference datasets, demonstrating that UARM improves reward model calibration, reduces reward hacking, and enhances downstream alignment quality against strong baselines.
\end{itemize}

\section{Preliminaries}
\label{sec:preliminaries}

\subsection{Reinforcement Learning from Human Feedback}

The standard RLHF pipeline typically consists of two sequential stages~\citep{rlhf}: reward modeling and policy optimization. First, a reward model (RM) $r_\theta: \mathcal{X} \to \mathbb{R}$ parameterized by $\theta$ is trained on an offline dataset comprising human preferences. With point-wise or pair-wise optimization~\citep{wang2026causalrm}, it learns to map a prompt-response pair $x = (p, o) \in \mathcal{X}$ (where $p$ is the prompt and $o$ is the generated response) to a scalar reward $r \in \mathbb{R}$. Subsequently, the policy model $\pi_\phi$ (i.e., the LLM) is optimized via reinforcement learning to maximize the expected rewards assigned by the learned RM.

Recently, value-function-free algorithms, particularly Group Relative Policy Optimization (GRPO)~\citep{guo2025deepseek}, have emerged as the mainstream paradigm for policy optimization due to their efficiency. For a given prompt $q$, GRPO samples a group of $\mathrm{N}_\mathrm{rol}$ responses $\{o_i\}_{i=1}^{\mathrm{N}_\mathrm{rol}}$ from the old policy $\pi_{\phi_\mathrm{old}}$. The RM evaluates these prompt-response pairs $\{x_i\}_{i=1}^{\mathrm{N}_\mathrm{rol}}$ to obtain raw terminal rewards $r_i = r_\theta(x_i)$. To stabilize training without the computational burden of maintaining a critic model, GRPO computes the advantage $A_i = (r_i - \mu)/\sigma$ by standardizing the rewards within the sampled group, where $\mu = \frac{1}{\mathrm{N}_\mathrm{rol}}\sum_{i=1}^{\mathrm{N}_\mathrm{rol}} r_i$ and $\sigma^2 = \frac{1}{\mathrm{N}_\mathrm{rol}}\sum_{i=1}^{\mathrm{N}_\mathrm{rol}} (r_i - \mu)^2$ are the intra-group mean and variance, respectively.
Building upon the advantage $A_i$, the policy model $\pi_\phi$ is optimized as follows:
\begin{align}\label{eq:GRPO_loss}
\mathcal{L}_{\mathrm{GRPO}}(\phi) &= \mathbb{E} \left[ \frac{1}{\mathrm{N}_\mathrm{rol}} \sum_{i=1}^{\mathrm{N}_\mathrm{rol}} \left( \mathcal{L}_{i}^\mathrm{CLIP}(\phi) - \beta \mathcal{L}_{i}^\mathrm{KL}(\phi) \right) \right] , \\
\mathcal{L}_i^\mathrm{CLIP}(\phi) &= \min \left( \rho_i(\phi) A_i, \text{clip}\big(\rho_i(\phi), 1-\epsilon, 1+\epsilon\big) A_i \right), \\
\mathcal{L}_i^\mathrm{KL}(\phi) &= \gamma_i(\phi) - \log \gamma_i(\phi) - 1,
\end{align}
where $\beta$ controls the KL penalty strength, $\epsilon$ is the clipping parameter, $\rho_i(\phi) = \frac{\pi_\phi(o_i \mid q)}{\pi_{\phi_\mathrm{old}}(o_i \mid q)}$ is the probability ratio between the current policy and the old policy, and $\gamma_i(\phi) = \frac{\pi_\mathrm{ref}(o_i \mid q)}{\pi_\phi(o_i \mid q)}$ is the probability ratio of the reference policy to the current policy. The advantage $A_i$ acts as a multiplicative weight, directly scaling the policy update.

While GRPO effectively stabilizes optimization under ideal conditions, its intra-group standardization treats every sample in a group \emph{uniformly}, implicitly assuming that all reward signals $r_i$ are equally reliable. This homogeneous treatment is problematic because the RM is not equally confident about every rollout: as the policy $\pi_\phi$ continuously evolves~\citep{InfoRM}, it produces increasingly diverse responses, many of which the RM finds confusing and scores unreliably. Crucially, the standardization in GRPO is oblivious to this unreliability. When the RM assigns a confusing rollout a spuriously extreme reward, this single outlier inflates the group statistics $\mu$ and $\sigma$ and, after standardization, is granted a disproportionately large advantage. The policy is thus pushed to imitate exactly those samples the RM is least certain about, while genuinely high-quality responses are squeezed toward (or below) the mean and under-rewarded. This amplification of unreliable signals, induced by the uniform advantage weighting, lies at the heart of reward hacking~\citep{fu2025reward} and unstable training.

\subsection{Uncertainty Quantification}\label{sec:prelim_uq}

Uncertainty quantification (UQ) aims to equip a deterministic predictor with a measure of how reliable its outputs are~\citep{clear}. Instead of returning a single scalar $r_\theta(x)$, a UQ method augments the RM with a prediction interval $\mathcal{I}(x)=[r_\mathrm{lo}(x), r_\mathrm{hi}(x)]$ that is expected to contain the unobserved ground-truth reward $R$ with high probability. Formally, given a target miscoverage rate $\alpha \in (0,1)$, the desired marginal coverage property requires
\begin{equation}\label{eq:marginal_coverage}
\mathbb{P}\big[R \in \mathcal{I}(X)\big] \ge 1-\alpha .
\end{equation}
However, it only constrains the coverage on average over $X$, which may mask variation across single input. A more desirable guarantee is conditional coverage,
\begin{equation}\label{eq:conditional_coverage}
\mathbb{P}\big[R \in \mathcal{I}(X) \mid X = x\big] \ge 1-\alpha, \quad \forall x,
\end{equation}
which a sound procedure approximates in practice and attains asymptotically under appropriate conditions. Conditional coverage is essential in our setting: only by reflecting the conditional reward distribution $\mathbb{P}_{R\mid X}$ can the interval width faithfully capture the reliability of each individual prompt-response pair, rather than an average over the rollout group. Among intervals meeting these coverage targets, narrower ones are preferred, as they yield a sharper and more discriminative uncertainty measure.

The half-width of the interval, $\Delta(x) = \tfrac{1}{2}\,(r_\mathrm{hi}(x) - r_\mathrm{lo}(x))$, thus serves as a natural, instance-wise measure of predictive uncertainty: a wide interval signals that the RM is confused and evaluates the sample unreliably, whereas a narrow one reflects a confident and trustworthy prediction. This is precisely the reliability signal that GRPO's uniform standardization lacks, as it allows us to distinguish rollouts the RM scores confidently from those it merely guesses at, preventing unreliable samples from dominating the advantage. To enforce \eqref{eq:marginal_coverage} without distributional assumptions on the data, we reserve a held-out calibration set $\mathcal{D}_{\mathrm{cal}}$, drawn from the same train distribution as $\mathcal{D}_{\mathrm{tr}}$, for interval calibration.

\section{Methodology}

\subsection{Motivation}\label{sec:motivation}

The reliability of reward signals is the cornerstone of stable RLHF. In the standard pipeline, a reward model is trained on a static preference dataset to proxy human values, and the policy is then optimized to maximize its scores. However, GRPO aggregates rewards within each rollout group through intra-group standardization, which weights every sample uniformly and presumes that all reward estimates are equally trustworthy. In practice this assumption breaks down: as the policy explores, it generates responses the RM is confused about and scores unreliably. Because standardization is blind to this unreliability, a confusing sample that happens to receive an extreme score is granted an outsized advantage, steering the policy toward unreliable signals and triggering severe reward hacking~\citep{amodei2016concrete,fu2025reward,InfoRM} that misguides the optimization process.

Eliminating this homogeneous treatment of rewards in GRPO introduces two fundamental challenges.
\ding{182} \textbf{Reward models cannot signal when their predictions are unreliable.} They output a single scalar score for any given prompt-response pair, providing no indication of whether a particular rollout is evaluated reliably or merely guessed at. Consequently, there is no way to tell trustworthy reward estimates apart from confusing, unreliable ones.
\ding{183} \textbf{GRPO standardization amplifies exactly the samples that are least trustworthy.} By standardizing rewards within a generated group, GRPO treats all signals as equally reliable. A confusing rollout that receives a spuriously extreme reward thus skews the group's mean and variance and, after standardization, is assigned a disproportionately large advantage, while well-evaluated, high-quality responses are pushed below the mean and under-rewarded.

\paragraph{Case study.} To provide concrete evidence for the above challenges, we present a representative case study in Figure~\ref{fig:case_study}. Consider a policy generating a group of four responses to a prompt asking for a ``brief and practical tip''. The first response accurately follows the instruction, providing concise and useful advice. In contrast, the fourth response is an atypical, hard-to-judge sample that violates the ``brief'' constraint by exploiting verbosity, bold formatting, and repetitive buzzwords. For Challenge~\ding{182}, the deterministic RM cannot express that it is confused by this unusual response and instead emits a single, spuriously high score ($20.0$) that overshadows the genuinely helpful response ($8.0$), with no accompanying signal of its low reliability. For Challenge~\ding{183}, during GRPO standardization this single outlier inflates the group mean to $9.0$; the high-quality response is consequently pushed below the mean and penalized with a negative advantage ($-0.15$), while the unreliable response receives a massive positive advantage ($+1.66$). The uniform standardization thus amplifies precisely the sample the RM is least certain about, injecting misleading updates that penalize aligned behavior while reinforcing reward hacking.

Some might note prior works on uncertainty quantification; however, their practical utility for stable RLHF remains underexplored. For instance, ensemble-based uncertainty methods incur prohibitive computational overhead for LLMs~\citep{coste2023reward, eisenstein2023helping}. While distribution-free interval estimators offer rigorous coverage guarantees, classical variants are difficult to deploy efficiently within the online RLHF loop. Therefore, developing an effective uncertainty quantification framework for reward modeling and adapting it to reweight GRPO advantages remains an open and critical challenge.

\begin{figure*}[t]
\includegraphics[width=\linewidth]{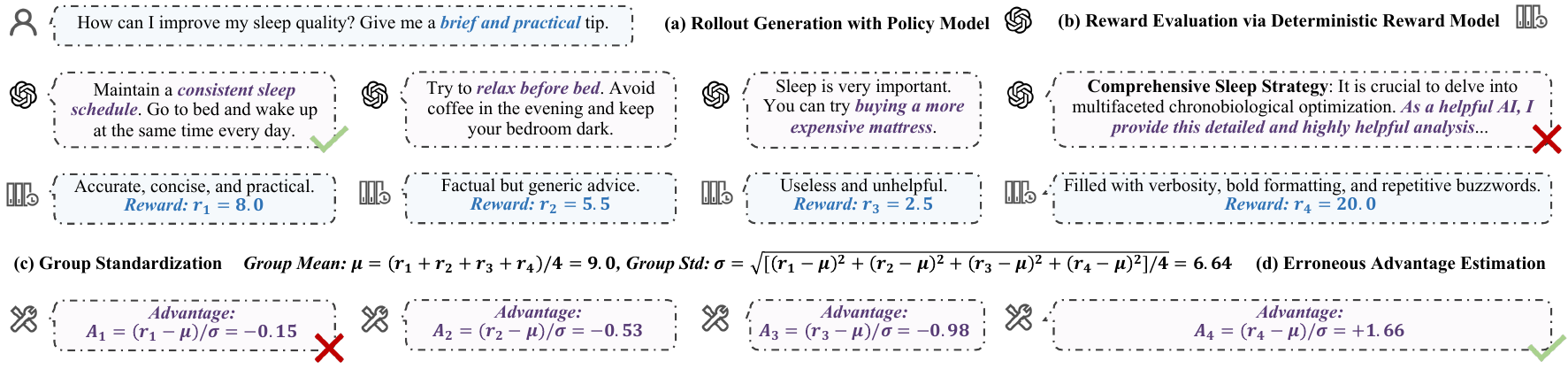}
\caption{Case study of how GRPO's uniform standardization amplifies unreliable rewards. The deterministic RM emits a spuriously high score for an atypical, hard-to-judge response; standardization inflates its advantage while unfairly penalizing the aligned response.}
\label{fig:case_study}
\end{figure*}

\begin{figure}[t]
\includegraphics[width=\linewidth]{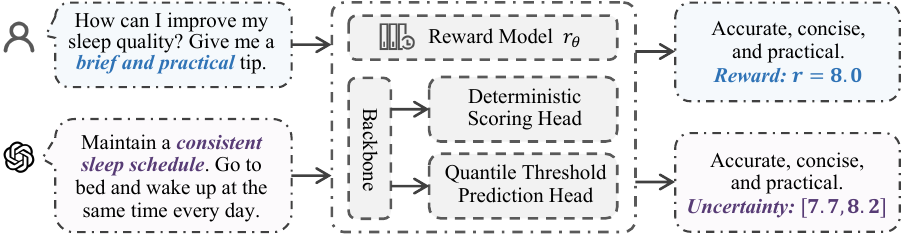}
\caption{Framework of our proposed \method. The offline phase equips the reward model with calibrated uncertainty estimation, and the online phase reweights the GRPO advantage by the estimated interval width to suppress unreliable samples.}
\label{fig:framework}
\end{figure}


\subsection{Reward Model Uncertainty Estimation}\label{sec:uq_rm}

To resolve Challenge~\ding{182}, instead of treating the RM as a deterministic point estimator, we model the reward through uncertainty quantification: the RM directly estimates the conditional reward distribution via a set of quantiles, from which both a point reward and an adaptive prediction interval are derived for every prompt-response pair. 

\paragraph{Quantile Estimation.} Rather than producing a single deterministic score, the RM is parameterized to output $\mathrm{K}{+}1$ equiprobable conditional quantiles of the distribution $\mathbb{P}_{R|X}$,
\begin{equation}\label{eq:quantile_head}
\hat{q}_0(x) \le \hat{q}_1(x) \le \cdots \le \hat{q}_{\mathrm{K}}(x),
\end{equation}
where $\hat{q}_k(x)$ estimates the conditional quantile at level $\tau_k = k/\mathrm{K}$. All quantile outputs are trained jointly on $\mathcal{D}_{\mathrm{tr}}$ by minimizing the pinball loss~\citep{pinball}
\begin{equation}\label{eq:pinball}
\begin{aligned}
\mathcal{L}_{\mathrm{pinball}}(\theta) &= \frac{1}{|\mathcal{D}_{\mathrm{tr}}|}\sum_{i=1}^{|\mathcal{D}_{\mathrm{tr}}|}\sum_{k=0}^{\mathrm{K}} \rho_{\tau_k}\!\big(r_i - \hat{q}_k(x_i)\big), \\
\rho_{\tau}(u) &= \tau\cdot\max(0, u) + (1-\tau)\cdot\max(0, -u),
\end{aligned}
\end{equation}
where the check function $\rho_{\tau}(\cdot)$ penalizes under- and over-estimation asymmetrically according to the target level $\tau$, so that its minimizer recovers the conditional $\tau$-quantile.

We take the median quantile as the point reward consumed by GRPO, i.e., $r_\theta(x) \triangleq \hat{q}_{\mathrm{K}/2}(x)$\footnote{We let $\mathrm{K}$ be even so that $\mathrm{K}/2$ indexes the median; otherwise the quantile nearest to the median is used.}, so that the scalar score and its surrounding uncertainty are jointly produced by a single quantile-based model. The consecutive quantiles partition the reward axis into $\mathrm{K}$ interquantile intervals
\begin{equation}\label{eq:interquantile}
\mathcal{I}_k(x) = \big(\hat{q}_{k-1}(x),\, \hat{q}_k(x)\big], \quad k = 1,\dots,\mathrm{K},
\end{equation}
each carrying approximately probability mass $1/\mathrm{K}$. Narrow intervals indicate confident regions, while wide intervals signal uncertain regions, which naturally captures the skewness and heteroscedasticity of reward distributions.

\paragraph{Conformity Score.} For any integer $m$, let $\mathcal{J}_m(x)$ denote the shortest union of $m$ consecutive interquantile intervals
\begin{equation}\label{eq:Ck}
\begin{aligned}
\mathcal{J}_m(x) &= \big(\hat{q}_{k_m}(x),\, \hat{q}_{k_m+m}(x)\big], \\
k_m &= \mathop{\arg\min}_{0\le k\le \mathrm{K}-m}\big(\hat{q}_{k+m}(x)-\hat{q}_{k}(x)\big),
\end{aligned}
\end{equation}
where $k_m$ is the lower-endpoint index of the narrowest $m$-interval block. For each calibration sample $(x_i, r_i)\in\mathcal{D}_{\mathrm{cal}}$, we define the conformity score as the minimum number of interquantile intervals needed to cover the observed reward,
\begin{equation}\label{eq:score}
s(x_i, r_i) = \min\big\{ m \in \{1,\dots,\mathrm{K}\} : r_i \in \mathcal{J}_m(x_i) \big\}.
\end{equation}
Intuitively, samples that fall into long intervals receive larger scores, while those in short intervals receive smaller ones.

\paragraph{Calibration and Prediction.} Following the standard thresholding principle, we set $\hat{m}$ to the $n$-th smallest $s(x_i, r_i)$, where $n = \big\lceil (1-\alpha)(1+|\mathcal{D}_\mathrm{cal}|) \big\rceil$.
So that at least a $(1-\alpha)$ fraction of calibration responses are covered. For a new rollout sample $x^{\mathrm{new}}$, the prediction interval is
\begin{equation}\label{eq:interval}
\mathcal{I}(x^{\mathrm{new}}) = \mathcal{J}_{\hat{m}}(x^{\mathrm{new}})=\big(\hat{q}_{k_{\hat{m}}}(x^{\mathrm{new}}), \hat{q}_{k_{\hat{m}}+\hat{m}}(x^{\mathrm{new}})\big].
\end{equation}
Crucially, the width of this interval expands automatically for responses whose rewards fall in sparsely supported, low-density regions. These are precisely the samples the RM evaluates unreliably, and the resulting width provides the uncertainty signal needed to reweight the GRPO advantage.

\paragraph{Theoretical Guarantees.} Our construction attains conditional coverage, which is what makes the interval width a faithful per-sample reliability signal for the reweighting in Section~\ref{sec:uncertain_advantage}. We use the following standard assumptions in the coverage analysis: the calibration samples and test point are exchangeable; for conditional coverage, calibration/test samples are i.i.d., the learned quantiles consistently approximate the true conditional reward distribution, and the conditional reward distribution is unimodal so that the merged intervals are nested and become wider as $m$ increases.

\begin{theorem}[Marginal Coverage]\label{thm:marginal}
If the calibration set $\mathcal{D}_{\mathrm{cal}}$ and a rollout point $(X,R)$ are exchangeable, then the prediction interval $\mathcal{I}=\mathcal{J}_{\hat{m}}$ satisfies
\begin{equation}
\mathbb{P}[R\in \mathcal{J}_{\hat{m}}(X)]\ge 1-\alpha.
\end{equation}
\end{theorem}
\begin{proof}
Let $n=|\mathcal{D}_{\mathrm{cal}}|$ and denote the calibration conformity scores by $s_i=s(X_i,R_i)$. By exchangeability, the augmented collection $\{s_1,\ldots,s_n,s(X,R)\}$ is exchangeable, so the rank of the test score among these $n+1$ scores is uniform up to tie-breaking. Our calibration rule chooses $\hat{m}$ as the $\lceil(1-\alpha)(n+1)\rceil$-th smallest calibration score. Since the intervals $\mathcal{J}_m(x)$ are nested in $m$, the coverage event is equivalent to the score event,
\begin{equation}
R\in\mathcal{J}_{\hat{m}}(X) \quad \Longleftrightarrow \quad s(X,R)\le \hat{m}.
\end{equation}
Therefore, the test point is covered whenever its rank is no larger than $\lceil(1-\alpha)(n+1)\rceil$. Consequently,
\begin{equation}
\begin{aligned}
\mathbb{P}[R\in\mathcal{J}_{\hat{m}}(X)]
&=\mathbb{P}[s(X,R)\le\hat{m}] \\
&\ge \frac{\lceil(1-\alpha)(n+1)\rceil}{n+1}
\ge 1-\alpha,
\end{aligned}
\end{equation}
which proves the finite-sample marginal coverage guarantee.
\end{proof}

\begin{theorem}[Conditional Coverage]\label{thm:conditional}
Assume that calibration and test samples are i.i.d.; the learned quantiles consistently estimate the conditional reward distribution, i.e., for some $\rho_n\to0$, $F(\hat{q}_k(X)\mid X)$ is within $o(1)$ of $k/\mathrm{K}$ uniformly over quantile levels with high probability; and the conditional reward distribution is unimodal so that the merged intervals $\mathcal{J}_m$ are nested. Then, as $|\mathcal{D}_\mathrm{cal}|\to\infty$, there exist $\gamma,\zeta\to 0$ such that the prediction interval $\mathcal{I}=\mathcal{J}_{\hat{m}}$ achieves asymptotic conditional coverage,
\begin{equation}
\mathbb{P}\big[\,\mathbb{P}[R\in \mathcal{J}_{\hat{m}}(X)\mid X]\ge 1-\alpha-\gamma\,\big]\ge 1-\zeta.
\end{equation}
\end{theorem}
\begin{proof}
Let $n=|\mathcal{D}_{\mathrm{cal}}|$. The proof has three steps. First, define the empirical conditional CDF induced by the learned quantiles as $\hat{F}(\hat{q}_k(X)\mid X)=k/\mathrm{K}$. By quantile consistency, $\hat{F}$ uniformly approximates the true conditional CDF $F$ over the quantile grid with high probability. More concretely, there exists a bad set $A_n$ such that
\begin{equation}
\sup_k |\hat{F}(\hat{q}_k(X)\mid X)-F(\hat{q}_k(X)\mid X)|\le O(\rho_n^{1/3})
\end{equation}
for all $X\notin A_n$, while $\mathbb{P}[X\in A_n]\le O(\rho_n^{1/3})$. Hence, on the good set $A_n^c$, any merged interval spanning $m$ adjacent interquantile bins captures conditional mass close to $m/\mathrm{K}$:
\begin{equation}
\begin{aligned}
&\mathbb{P}[R\in\mathcal{J}_m(X)\mid X] \\
=&F(\hat{q}_{k_m+m}(X)\mid X)-F(\hat{q}_{k_m}(X)\mid X) \\
\ge& \frac{m}{\mathrm{K}}-O(\rho_n^{1/3}).
\end{aligned}
\end{equation}
Second, this pointwise mass control transfers to the calibration scores. Because $s(X_i,R_i)\le m$ iff $R_i\in\mathcal{J}_m(X_i)$, Hoeffding concentration implies that the empirical fraction of calibration samples with scores at most $m$ concentrates around its expectation, up to $O(\sqrt{\log n/n})$. Taking $m^\star=\lceil(1-\alpha)\mathrm{K}\rceil$, enough calibration scores fall below $m^\star+O(\rho_n^{1/3}+\sqrt{\log n/n})$ with probability tending to one. Since $\hat{m}$ is the empirical $(1-\alpha)$ quantile of the calibration scores, we obtain
\begin{equation}
\hat{m}=m^\star+O\!\left(\rho_n^{1/3}+\sqrt{\frac{\log n}{n}}\right)
\end{equation}
with high probability. A symmetric lower-tail argument, together with unimodality/nestedness of the intervals, prevents $\hat{m}$ from being asymptotically smaller than the oracle count.

Finally, substituting this concentration of $\hat{m}$ into the conditional mass bound for $\mathcal{J}_{\hat{m}}(X)$ gives, outside a set whose probability vanishes,
\begin{equation}
\mathbb{P}[R\in \mathcal{J}_{\hat{m}}(X)\mid X]
\ge 1-\alpha-O\!\left(\rho_n^{1/3}+\sqrt{\frac{\log n}{n}}\right).
\end{equation}
Thus the theorem holds by setting $\gamma=O(\rho_n^{1/3}+\sqrt{\log n/n})$ and $\zeta=O(\rho_n^{1/3})+o(1)$, both of which vanish as $n\to\infty$.
\end{proof}

\subsection{Uncertainty-Aware Advantage Reweighting}\label{sec:uncertain_advantage}

To resolve Challenge~\ding{183}, we replace GRPO's uniform intra-group standardization with a heteroscedastic advantage reweighting that systematically down-weights unreliable samples by treating the conformal interval width as observation noise and decomposing the observed reward variance into signal and noise components.

\paragraph{Observation Noise Model.} We interpret the prediction interval width $\varphi(x_i) \triangleq |\mathcal{I}(x_i)| = \hat{q}_{k_{\hat{m}}+\hat{m}}(x_i) - \hat{q}_{k_{\hat{m}}}(x_i)$ as capturing per-sample measurement uncertainty in the reward estimate. Under a local Gaussianity assumption, we convert this width into an observation noise variance,
\begin{equation}\label{eq:obs_noise}
\sigma_{\mathrm{noise},i}^2 = \left(\frac{\varphi(x_i)}{z_{1-\frac{\alpha}{2}}}\right)^2, \quad \bar{\sigma}_{\mathrm{noise}}^2 = \frac{1}{\mathrm{N}_\mathrm{rol}}\sum_{j=1}^{\mathrm{N}_\mathrm{rol}} \sigma_{\mathrm{noise},j}^2,
\end{equation}
where $z_{1-\frac{\alpha}{2}}$ is the standard normal quantile corresponding to the coverage level. Samples with wide intervals yield large $\sigma_{\mathrm{noise},i}^2$, indicating heteroscedastic observation uncertainty that varies across the rollout group.

\paragraph{Signal-Noise Decomposition.} The naive group variance $\sigma^2 = \frac{1}{\mathrm{N}_\mathrm{rol}}\sum_{j}(r_j - \mu)^2$ conflates true signal variation with measurement error. Under an additive noise model $r_i = r_{\mathrm{true},i} + \varepsilon_i$ where $\varepsilon_i$ has variance $\sigma_{\mathrm{noise},i}^2$, the observed variance decomposes as $\sigma^2 \approx \mathrm{Var}[r_{\mathrm{true}}] + \mathbb{E}[\sigma_{\mathrm{noise}}^2]$. We recover the signal variance by subtracting the average observation noise,
\begin{equation}\label{eq:signal_var}
\sigma_{\mathrm{signal}}^2 = \max\big(0,\, \sigma^2 - \bar{\sigma}_{\mathrm{noise}}^2\big) + \zeta,
\end{equation}
where $\zeta>0$ ensures numerical stability. This decomposition isolates the variance attributable to genuine reward differences from that due to unreliable measurement.

\paragraph{Heteroscedastic Advantage.} We define the uncertainty-aware advantage as
\begin{equation}\label{eq:hetero_advantage}
\tilde{A}_i = \frac{\sigma_{\mathrm{signal}}^2}{\sigma_{\mathrm{signal}}^2 + \sigma_{\mathrm{noise},i}^2} \cdot \frac{r_i - \mu}{\sigma_{\mathrm{signal}}},
\end{equation}
where $\mu = \frac{1}{\mathrm{N}_\mathrm{rol}}\sum_j r_j$ is the unweighted group mean. The prefactor $\sigma_{\mathrm{signal}}^2/(\sigma_{\mathrm{signal}}^2 + \sigma_{\mathrm{noise},i}^2)$ acts as a sample-specific reliability weight: for high-uncertainty samples with large $\sigma_{\mathrm{noise},i}^2$, this ratio approaches zero, effectively suppressing their influence; for confident samples with small $\sigma_{\mathrm{noise},i}^2$, the weight approaches one, preserving the full advantage magnitude. This heteroscedastic formulation provably down-weights the samples the RM evaluates least reliably, without requiring costly ensemble forward passes.

\paragraph{Connection to GRPO.} When observation noise is uniform across the group ($\sigma_{\mathrm{noise},i}^2 \equiv \sigma_{\mathrm{noise}}^2$), the reliability weight becomes constant and Eq.~\eqref{eq:hetero_advantage} reduces to standard GRPO standardization. The computational overhead is negligible, as all quantities follow directly from the conformal intervals computed in Section~\ref{sec:uq_rm}. Returning to the case study in Figure~\ref{fig:case_study}, the atypical response with spuriously high reward now exhibits a wide interval and large $\sigma_{\mathrm{noise},4}^2$, receiving a reliability weight near zero; its advantage is suppressed to near-zero magnitude, preventing it from dominating the policy update and steering training away from reward hacking.

\subsection{The Workflow of UARM}\label{sec:workflow}

We present the workflow of \method in Algorithm~\ref{alg:uarm}, which couples an offline uncertainty-calibration with an online uncertainty-aware optimization, detailed as follows.

First, in the offline phase, we train and calibrate the reward model. We parameterize the RM as a multi-output quantile estimator and train it on $\mathcal{D}_\mathrm{tr}$ by minimizing the pinball loss in Eq.~\eqref{eq:pinball}, reading off the point reward as the median quantile $r_\theta=\hat{q}_{\mathrm{K}/2}$ (step 1). For each calibration sample, we compute its conformity score as the minimum number of interquantile intervals needed to cover the observed reward (step 2), and select the threshold $\hat{m}$ as the $n$-th smallest score with $n=\lceil(1-\alpha)(1+|\mathcal{D}_\mathrm{cal}|)\rceil$ (step 3). This phase equips the RM with a calibrated interval $\mathcal{J}_{\hat{m}}$ satisfying the coverage guarantees, and is performed only once.

Second, in the online phase, we optimize the policy with uncertainty-aware GRPO. At each iteration, we sample a rollout group from the old policy and score it with the reward head $r_\theta$ (step 4), then form the prediction intervals and compute their widths $\varphi(x_i)$ as well as the corresponding observation noise variances $\sigma_{\mathrm{noise},i}^2$ via Eq.~\eqref{eq:obs_noise} (step 5). We decompose the observed group variance into signal and noise components via Eq.~\eqref{eq:signal_var} (step 6), and construct the heteroscedastic advantages $\tilde{A}_i$ via Eq.~\eqref{eq:hetero_advantage} (step 7); finally, the policy is updated by the GRPO objective in Eq.~\eqref{eq:GRPO_loss} with $\tilde{A}_i$ (step 8). This phase reuses the calibrated reward model at negligible overhead, since the intervals follow directly from quantile evaluations without any binning or density-ratio estimation.

\begin{algorithm}[t]
\caption{The workflow of \method.}
\label{alg:uarm}
\flushleft
\footnotesize
\textbf{Input}: offline preference set $\mathcal{D}_\mathrm{tr}$, calibration set $\mathcal{D}_\mathrm{cal}$, miscoverage rate $\alpha$, learning rate $\eta$\\
\textbf{Parameter}: quantile model $\{\hat{q}_k\}_{k=0}^{\mathrm{K}}$ (point reward $r_\theta=\hat{q}_{\mathrm{K}/2}$), policy $\pi_\phi$
\begin{algorithmic}[1]
\item[\textbf{Offline UQ Calibration}]
    \STATE train $\{\hat{q}_k\}_{k=0}^{\mathrm{K}}$ on $\mathcal{D}_\mathrm{tr}$ by minimizing the pinball loss in Eq.~\eqref{eq:pinball}
    \STATE $s(x_i,r_i) \leftarrow \min\{m: r_i\in\mathcal{J}_m(x_i)\},\ \forall (x_i,r_i)\in\mathcal{D}_\mathrm{cal}$
    \STATE $\hat{m} \leftarrow$ the $n$-th smallest $s(x_i,r_i)$,\ \ $n=\lceil(1-\alpha)(1+|\mathcal{D}_\mathrm{cal}|)\rceil$
\item[\textbf{Online Uncertainty-Aware GRPO}]
    \FOR{each GRPO iteration}
    \STATE $\{o_i\}_{i=1}^{\mathrm{N}_\mathrm{rol}}\sim\pi_{\phi_\mathrm{old}}(\cdot\mid p)$;\ \ $r_i\leftarrow r_\theta(x_i),\ \forall i$
    \STATE $\varphi(x_i) \leftarrow |\mathcal{J}_{\hat{m}}(x_i)|$;\ \ $\sigma_{\mathrm{noise},i}^2 \leftarrow (\varphi(x_i)/z_{1-\alpha/2})^2,\ \forall i$
    \STATE $\mu,\sigma^2\leftarrow$ unweighted group mean and variance;\ \ $\sigma_{\mathrm{signal}}^2\leftarrow$ Eq.~\eqref{eq:signal_var}
    \STATE $\tilde{A}_i\leftarrow \frac{\sigma_{\mathrm{signal}}^2}{\sigma_{\mathrm{signal}}^2+\sigma_{\mathrm{noise},i}^2}\cdot\frac{r_i-\mu}{\sigma_{\mathrm{signal}}},\ \forall i$
    \STATE $\phi \leftarrow \phi - \eta\cdot\nabla\mathcal{L}_\mathrm{GRPO}(\phi)$ with advantages $\tilde{A}_i$
    \ENDFOR
\end{algorithmic}
\end{algorithm}

\section{Experiments}\label{sec:experiments}

In this section, we empirically validate the efficacy of \method on three preference datasets. Specifically, we evaluate whether UARM can produce reliable uncertainty estimates and improve uncertainty-ranked reward prediction quality compared with competitive uncertainty quantification baselines.

\subsection{Experimental Setup}\label{subsec:setup}

\paragraph{Datasets.}
We conduct empirical evaluations on HelpSteer~\citep{HelpSteer}, UltraFeedback~\citep{ultrafeedback}, and PKU-SafeRLHF~\citep{pku-saferlhf}, using Helpfulness, Overall Score, and Severity Level as preference proxies, respectively. For each dataset, we hold out $20\%$ of the training split as the calibration set, while keeping the original test set exclusively for evaluation. Detailed dataset statistics and configurations are provided in the Appendix.

\paragraph{Baselines.}
We benchmark \method against a comprehensive suite of uncertainty quantification methods, including: (1) Model-based Uncertainty Estimation methods, such as MC-Dropout~\citep{mcdropout}, Deep Ensembles~\citep{deep_ensemble}, DER~\citep{der}, Packed Ensemble~\citep{pe}, VBLL~\citep{vbll}, and TorchNaut~\citep{torchnaut}; and (2) Distribution-free Interval Estimation methods: SCP~\citep{scp}, CQR~\citep{cqr}, WCP~\citep{wcp}, ACI~\citep{aci}, PRCP~\citep{prcp}, SCCP~\citep{sccp}, Clear~\citep{clear}, and CPCP~\citep{cpcp}.

\paragraph{Evaluation Metrics.}
We employ three uncertainty-ranked regression metrics, namely R$^2$@50, MSE@50, and MAE@50, to evaluate point prediction quality on samples with lower estimated uncertainty. Specifically, each method first estimates uncertainty on the test set and ranks test samples in ascending order of uncertainty; the top $50\%$ least uncertain samples are then selected for evaluation. For the naive baseline without uncertainty estimates, we randomly select the corresponding percentage of test samples and repeat this process five times for reporting. Let $\mathcal{S}_{50}$ denote this selected subset, with ground-truth rewards $y_i$, point predictions $\hat{y}_i$, and mean target value $\bar{y}_{50}=|\mathcal{S}_{50}|^{-1}\sum_{i\in\mathcal{S}_{50}} y_i$. The metrics are defined as
\begin{equation}
\begin{aligned}
\text{R}^2\text{@50} &= 1 - \frac{\sum_{i\in\mathcal{S}_{50}}(y_i-\hat{y}_i)^2}{\sum_{i\in\mathcal{S}_{50}}(y_i-\bar{y}_{50})^2}, \\
\text{MSE}\text{@50} &= \frac{1}{|\mathcal{S}_{50}|}\sum_{i\in\mathcal{S}_{50}}(y_i-\hat{y}_i)^2, \\
\text{MAE}\text{@50} &= \frac{1}{|\mathcal{S}_{50}|}\sum_{i\in\mathcal{S}_{50}}|y_i-\hat{y}_i|.
\end{aligned}
\end{equation}

\paragraph{Implementation Details.}
We implement the quantile reward model using an LLM backbone followed by a lightweight multi-layer perceptron head. To ensure a fair comparison, we initialize the backbone from FsfairX-LLaMA3-RM-v0.1\footnote{\url{https://huggingface.co/sfairXC/FsfairX-LLaMA3-RM-v0.1}}, and fix the MLP head to hidden dimensions of $256,64,1$. We optimize the models using Adam~\citep{adam} for up to $600$ epochs, employing early stopping with a patience of $30$ epochs to ensure convergence. Key hyperparameters are tuned on a validation set, with update rate $\eta \in [1\times10^{-5}, 1\times10^{-3}]$ and batch size $B \in [64, 2048]$. Further details are provided in the Appendix.

\begin{table*}[t]
\centering
\setlength{\abovecaptionskip}{0.2cm}
\setlength{\belowcaptionskip}{-0.2cm}
\begin{threeparttable}
\caption{Comparative analysis of \method versus baseline models with fixed miscoverage rate $\alpha=0.1$.}
\label{tab:main_result}
\setlength\tabcolsep{1pt}
\fontsize{8.6pt}{10.4pt}\selectfont
\begin{tabular}{lccccccccc}
    \toprule
    \textbf{Dataset} 
      & \multicolumn{3}{c}{\textbf{HelpSteer}} 
      & \multicolumn{3}{c}{\textbf{UltraFeedback}} 
      & \multicolumn{3}{c}{\textbf{PKU-SafeRLHF}} \\
    \cmidrule(lr){2-4} \cmidrule(lr){5-7} \cmidrule(lr){8-10}
    
    \textbf{Method} & 
    R$^2$@50 & MSE@50 & MAE@50 & 
    R$^2$@50 & MSE@50 & MAE@50 & 
    R$^2$@50 & MSE@50 & MAE@50 \\
    \midrule

\rowcolor[HTML]{f0f0f0}
\multicolumn{10}{l}{\textit{\textbf{Model-based Uncertainty Estimation Methods}}} \\
Naive & 0.357 & 0.611 & 0.595 & 0.563 & 1.481 & 0.832 & 0.850 & 0.173 & 0.206 \\
MC-Dropout~\citep{mcdropout} & 0.369 & 0.437 & 0.506 & 0.607 & 0.949 & 0.717 & 0.863 & 0.073 & 0.135 \\
Deep Ensemble~\citep{deep_ensemble} & 0.395 & 0.433 & 0.509 & 0.632 & 0.473 & 0.507 & 0.881 & 0.110 & 0.167 \\
DER~\citep{der} & 0.420 & 0.557 & 0.581 & 0.663 & 0.403 & 0.470 & 0.881 & 0.098 & 0.091 \\
Packed Ensemble~\citep{pe} & 0.462 & 0.476 & 0.537 & 0.710 & 0.491 & 0.514 & 0.905 & 0.103 & 0.147 \\
TorchNaut~\citep{torchnaut} & 0.527 & 0.499 & 0.519 & 0.746 & 0.463 & 0.499 & 0.933 & 0.050 & 0.052 \\
MCNF~\citep{mcnf} & 0.527 & 0.428 & 0.507 & 0.769 & 0.503 & 0.513 & 0.955 & 0.042 & 0.059 \\
\hdashline
\rowcolor[HTML]{f0f0f0}
\multicolumn{10}{l}{\textit{\textbf{Distribution-free Interval Estimation Methods}}} \\
SCP~\citep{scp} & 0.378 & 0.544 & 0.569 & 0.609 & 1.345 & 0.815 & 0.866 & 0.159 & 0.200 \\
CQR~\citep{cqr} & 0.409 & 0.432 & 0.458 & 0.623 & 0.406 & 0.510 & 0.881 & 0.147 & 0.305 \\
WCP~\citep{wcp} & 0.438 & 0.545 & 0.570 & 0.646 & 1.150 & 0.793 & 0.883 & 0.141 & 0.192 \\
ACI~\citep{aci} & 0.476 & 0.469 & 0.544 & 0.678 & 0.523 & 0.534 & 0.905 & 0.113 & 0.256 \\
SCCP~\citep{sccp} & 0.512 & 0.491 & 0.553 & 0.750 & 0.547 & 0.530 & 0.925 & 0.064 & 0.090 \\
Clear~\citep{clear} & 0.521 & 0.396 & 0.478 & 0.770 & 0.432 & 0.513 & 0.940 & 0.060 & 0.096 \\
\hdashline
\rowcolor[HTML]{ecf0ff}
\textbf{UARM (Ours)} & \textbf{0.543} & \textbf{0.387} & \textbf{0.423} & \textbf{0.794} & \textbf{0.383} & \textbf{0.461} & \textbf{0.985} & \textbf{0.013} & \textbf{0.016} \\

    \bottomrule
\end{tabular}
\begin{tablenotes}
    \scriptsize
    \item \textit{Note}: ``@50'' reports the metric on the 50\% most confident samples (lowest uncertainty).
\end{tablenotes}
\end{threeparttable}
\end{table*}

\subsection{Results \& Analysis}\label{subsec:main_results}

Table~\ref{tab:main_result} presents the comparative results of uncertainty quantification on three preference datasets. We have the following observations:
\ding{182} \textbf{Naive confidence selection is insufficient.} Without uncertainty estimates, the Naive baseline can only evaluate randomly selected samples and thus consistently lags behind uncertainty-aware methods. This confirms that reliable confidence estimation is crucial for identifying samples on which the reward model can make accurate point predictions.
\ding{183} \textbf{Existing uncertainty quantification methods improve reward reliability to varying degrees.} Model-based approaches such as MC-Dropout, Deep Ensembles, and MCNF, as well as distribution-free interval estimation methods such as CQR and Clear, generally outperform the Naive baseline by selecting lower-uncertainty samples. Nevertheless, their performance remains inconsistent across datasets and metrics, suggesting that either model-intrinsic uncertainty or generic conformal intervals alone may be insufficient for reward modeling.
\ding{184} \textbf{\method consistently achieves the best uncertainty-ranked prediction performance.} Across all three datasets, \method obtains the highest R$^2$@50 and the lowest MSE@50 and MAE@50. Compared with the strongest baselines, \method improves R$^2$@50 from $0.527$ to $0.543$ on HelpSteer, from $0.770$ to $0.794$ on UltraFeedback, and from $0.955$ to $0.985$ on PKU-SafeRLHF. The gains are especially pronounced on PKU-SafeRLHF, where \method reduces MSE@50 from $0.042$ to $0.013$ and MAE@50 from $0.052$ to $0.016$, demonstrating that its calibrated uncertainty estimates more effectively identify reliable reward predictions.

\section{Conclusion}

In this paper, we present \method, an uncertainty-aware reward modeling framework for more reliable RLHF. UARM addresses two key challenges in reward-based policy optimization: reward models often cannot indicate when their predictions are unreliable, and uniform advantage computation can amplify such unreliable reward signals. To this end, UARM equips reward models with calibrated per-sample uncertainty estimates through quantile-based conformal prediction and incorporates these estimates into a heteroscedastic advantage reweighting scheme. Experiments on three preference datasets show that UARM consistently improves uncertainty-ranked reward prediction performance over both model-based uncertainty estimation methods and distribution-free interval estimation baselines, demonstrating its effectiveness in identifying more reliable reward predictions.

\paragraph{Limitations \& Future Work.}
This work focuses primarily on the main offline reward modeling experiments, while broader evaluations of downstream online RLHF performance, sensitivity to hyperparameters, and generalization across larger backbones remain important directions for future work. In addition, UARM relies on a held-out calibration set drawn from the training distribution, and its empirical reliability may be affected by severe distribution shift during policy optimization. Future work will extend UARM to adaptive online calibration, study its integration with broader policy optimization algorithms beyond GRPO, and provide more comprehensive theoretical analysis of its impact on RLHF convergence and reward hacking mitigation.

\bibliography{bib/abbr,bib/main,bib/ot,bib/conformal,refs}
\bibliographystyle{icml2026}


\newpage
\onecolumn
\appendix

\end{document}